\documentclass{article}

\usepackage{algorithm, algorithmic}
\usepackage{amsmath, amsthm, amssymb}
\usepackage{fullpage}
\usepackage{natbib}
\usepackage{times}

\DeclareMathOperator*{\Exp}{\mathbf{E}}
\DeclareMathOperator{\Regret}{Regret}

\newcommand{\field}[1]{\mathbb{#1}}
\newcommand{\R}{\field{R}}

\newcommand{\Var}{\mathrm{Var}}

\newtheorem{theorem}{Theorem}
\newtheorem{lemma}[theorem]{Lemma}
\newtheorem{corollary}[theorem]{Corollary}

\begin{document}

\title{Optimal Non-Asymptotic Lower Bound \\ on the \\ Minimax Regret of Learning with Expert Advice}
\author{
\begin{tabular}{c@{\hskip 1in}c}
Francesco Orabona & David Pal \\
francesco@orabona.com & dpal@yahoo-inc.com \\
\end{tabular}
\\\\
Yahoo Labs \\
New York, NY, USA
}

\maketitle

\begin{abstract}
We prove non-asymptotic lower bounds on the expectation of the maximum of $d$
independent Gaussian variables and the expectation of the maximum of $d$
independent symmetric random walks. Both lower bounds recover the optimal
leading constant in the limit.  A simple application of the lower bound for
random walks is an (asymptotically optimal) non-asymptotic lower bound on the minimax regret of online
learning with expert advice.
\end{abstract}

\section{Introduction}

Let $X_1, X_2, \dots, X_d$ be i.i.d. Gaussian random variables $N(0,\sigma^2)$.
It easy to prove that (see Appendix~\ref{section:upper-bounds})
\begin{equation}
\label{equation:upper-bound-on-maximum-of-gaussians}
\Exp \left[ \max_{1 \le i \le d} X_i \right] \le \sigma \sqrt{2 \ln d} \qquad \text{for any $d \ge 1$} \; .
\end{equation}
It is also well known that
\begin{equation}
\label{equation:limit-maximum-of-gaussians}
\lim_{d \to \infty} \frac{\Exp \left[ \max_{1 \le i \le d} X_i \right]}{\sigma \sqrt{2 \ln d}} = 1 \; .
\end{equation}
In section~\ref{section:maximum-of-gaussians}, we prove a non-asymptotic
$\Omega(\sigma \sqrt{\log d})$ lower bound on $\Exp[\max_{1 \le i \le d} X_i]$.
The leading term of the lower bound is asymptotically $\sqrt{2 \ln d}$. In
other words, the lower bound
implies~\eqref{equation:limit-maximum-of-gaussians}.

Discrete analog of a Gaussian random variable is the symmetric random walk.
Recall that a random walk $Z^{(n)}$ of length $n$ is a sum $Z^{(n)} = Y_1 + Y_2
+ \dots + Y_n$ of $n$ i.i.d. Rademacher variables, which have probability
distribution $\Pr[Y_i = +1] = \Pr[Y_i = -1] = 1/2$. We consider $d$ independent
symmetric random walks $Z^{(n)}_1, Z^{(n)}_2, \dots, Z^{(n)}_d$ of length $n$.
Analogously to \eqref{equation:upper-bound-on-maximum-of-gaussians}, it is easy
to prove that (see Appendix~\ref{section:upper-bounds})
\begin{equation}
\label{equation:upper-bound-on-maximum-of-random-walks}
\Exp \left[ \max_{1 \le i \le d} Z^{(n)}_i \right] \le \sqrt{2 n \ln d} \qquad \text{for any $n \ge 0$ and any $d \ge 1$}\; .
\end{equation}
Note that $\sigma^2$ in \eqref{equation:upper-bound-on-maximum-of-gaussians} is
replaced by $\Var(Z^{(n)}_i) = n$. By central limit theorem
$\frac{Z^{(n)}_i}{\sqrt{n}}$ as $n \to \infty$ converges in distribution to
$N(0,1)$. From this fact, it possible to prove the analog of
\eqref{equation:limit-maximum-of-gaussians},
\begin{equation}
\label{equation:limit-maximum-of-random-walks}
\lim_{d \to \infty} \lim_{n \to \infty} \frac{\Exp\left[ \max_{1 \le i \le d} Z^{(n)}_i \right]}{\sqrt{2 n \ln d}} = 1 \; .
\end{equation}
We prove a non-asymptotic $\Omega(\sqrt{n \log d})$ lower bound on $\Exp\left[
\max_{1 \le i \le d} Z^{(n)}_i \right]$.  Same as for the Gaussian case, the
leading term of the lower bound is asymptotically $\sqrt{2 n \ln d}$
matching~\eqref{equation:limit-maximum-of-random-walks}.

In section~\ref{section:experts}, we show a simple application of the lower
bound on $\Exp\left[\max_{1 \le i \le d} Z^{(n)}_i \right]$ to the problem of
learning with expert advice.  This problem was extensively studied in the
online learning literature; see~\citep{Cesa-BianchiL06}.  Our bound is optimal
in the sense that for large $d$ and large $n$ it recovers the right leading
constant.

\section{Maximum of Gaussians}
\label{section:maximum-of-gaussians}

Crucial step towards lower bounding $\Exp \left[ \max_{1 \le i \le d} X_i
\right]$ is a good lower bound on the tail $\Pr[X_i \ge x]$ of a single
Gaussian. The standard way of deriving such bounds is via bounds on the so-called
Mill's ratio.  Mill's ratio of a random variable $X$ with density function
$f(x)$ is the ratio $\frac{\Pr[X > x]}{f(x)}$.\footnote{Mill's ratio has
applications in economics. A simple is problem where Mill's ratio shows up is
the problem of setting optimal price for a product.  Given a distribution
prices that customers are willing to pay, the goal is to choose the price that
brings the most revenue.} It clear that a lower bound on the Mill's ratio
yields a lower bound on the tail $\Pr[X > x]$.

Without loss of generality it suffices to lower bound the Mill's ratio of $N(0,1)$,
since Mill's ratio of $N(0,\sigma^2)$ can be obtained by rescaling.  Recall that
probability density of $N(0,1)$ is $\phi(x) = \frac{1}{\sqrt{2 \pi}}
e^{-x^2/2}$ and its cumulative distribution function is $\Phi(x) =
\int_{-\infty}^x \frac{1}{\sqrt{2 \pi}} e^{-t^2/2} dt$.  The Mill's ratio for
$N(0,1)$ can be expressed as $\frac{1 - \Phi(x)}{\phi(x)}$. A lower bound on
Mill's ratio of $N(0,1)$ was proved by~\cite{Boyd-1959}.
\begin{lemma}[Mill's ratio for standard Gaussian~\citep{Boyd-1959}]
\label{lemma:boyd}
For any $x \ge 0$,
$$
\frac{1 - \Phi(x)}{\phi(x)}
= \exp\left(\frac{x^2}{2}\right) \int_x^{\infty} \exp\left(-\frac{t^2}{2}\right) dt
\ge \frac{\pi}{(\pi-1)x+\sqrt{x^2+2 \pi}}
\ge \frac{\pi}{\pi x + \sqrt{2\pi}} \; .
$$
\end{lemma}
The second inequality in Lemma~\ref{lemma:boyd} is our simplification of Boyd's
bound.  It follows by setting $a=\sqrt{2 \pi}$ and $b=x$. By a simple algebra
it is equivalent to the inequality $a + b \ge \sqrt{a^2 + b^2}$ which holds for
any $a,b \ge 0$.
\begin{corollary}[Lower Bound on Gaussian Tail]
Let $X \sim N(0, \sigma^2)$ and $x \ge 0$. Then,
$$
\Pr[X \ge x] \ge \exp\left(-\frac{x^2}{2 \sigma^2}\right) \frac{1}{\sqrt{2\pi}\frac{x}{\sigma}+2}.
$$
\end{corollary}
\begin{proof}
We have
\begin{align*}
\Pr[X \ge x]
& = \frac{1}{\sigma \sqrt{2 \pi}} \int_x^{\infty} \exp \left( -\frac{t^2}{2 \sigma^2} \right) d t \\
& = \frac{1}{\sqrt{2 \pi}} \int_\frac{x}{\sigma}^{\infty} \exp \left( -\frac{t^2}{2} \right) d t \\
& \ge \frac{1}{\sqrt{2 \pi}} \exp\left(-\frac{x^2}{2 \sigma^2}\right) \frac{\pi}{\pi\frac{x}{\sigma}+\sqrt{2 \pi}} & \text{(by Lemma~\ref{lemma:boyd})} \; .
\end{align*}
\end{proof}

Equipped with the lower bound on the tail, we prove a lower bound on the maximum of Gaussians.
\begin{theorem}[Lower Bound on Maximum of Independent Gaussians]
\label{theorem:maximum-of-gaussians}
Let $X_1, X_2, \dots, X_d$ be independent Gaussian random variables $N(0,\sigma^2)$. For any $d \ge 2$,
\begin{align}
\Exp \left[\max_{1 \le i \le d} X_i\right]
& \ge \sigma \left(1 - \exp\left(-\frac{\sqrt{\ln d}}{6.35}\right)\right) \left(\sqrt{2 \ln d - 2 \ln \ln d} +\sqrt\frac{2}{\pi}\right) -\sqrt{\frac{2}{\pi}} \sigma \label{equation:maximum-of-gaussians-lower-bound-1} \\
& \ge 0.13 \sigma \sqrt{\ln d} - 0.7 \sigma \label{equation:maximum-of-gaussians-lower-bound-2} \; .
\end{align}
\end{theorem}
\begin{proof}
Let $A$ be the event that at least one of the $X_i$ is greater than $C \sigma
\sqrt{\ln d}$ where $C = C(d) = \sqrt{2 - \frac{2 \ln \ln d}{\ln d}}$. We
denote by $\overline{A}$ the complement of this event. We have
\begin{align}
\Exp \left[ \max_{1 \le i \le d} X_i \right]
& = \Exp \left[ \max_{1 \le i \le d} X_i ~ \middle|~ A \right] \cdot \Pr[A] + \Exp \left[ \max_{1 \le i \le d} X_i ~\middle|~ \overline{A} \right] \cdot \Pr \left[ \overline{A} \right] \notag \\
& \ge \Exp \left[ \max_{1 \le i \le d} X_i ~\middle|~ A \right] \cdot \Pr[A] + \Exp \left[ X_1 ~\middle|~ \overline{A} \right] \cdot \Pr[\overline{A}] \notag \\
& = \Exp \left[ \max_{1 \le i \le d} X_i~\middle|~ A \right] \cdot \Pr[A] + \Exp \left[ X_1~|~ X_1 \le C \sigma \sqrt{\ln d} \right] \cdot \Pr[\overline{A}] \notag \\
& \ge \Exp \left[ \max_{1 \le i \le d} X_i~\middle|~ A \right] \cdot \Pr[A] + \Exp[X_1~|~ X_1 \le 0] \cdot \Pr[\overline{A}] \notag \\
& \ge \Exp \left[ \max_{1 \le i \le d} X_i~\middle|~ A \right] \cdot \Pr[A] - \sigma \sqrt{\frac{2}{\pi}} \cdot \Pr[\overline{A}] \notag \\
& \ge C \sigma \sqrt{\ln d} \cdot \Pr[A] - \sigma \sqrt{\frac{2}{\pi}} (1 - \Pr[A]) \notag \\
& = \sigma \left(C\sqrt{\ln d} + \sqrt{\frac{2}{\pi}}\right) \Pr[A] -  \sigma \sqrt{\frac{2}{\pi}} \label{equation:maximum-of-gaussians-lower-bound-3}
\end{align}
where we used that $\Exp[X_1 ~|~ X_1 \le 0] = \frac{1}{\Pr[X_1 \le 0]} \int_{-\infty}^0 \frac{x}{\sigma \sqrt{2\pi}} \exp \left(- \frac{x^2}{2\sigma^2} \right) = - \sigma \sqrt{\frac{2}{\pi}}$.

It remains to lower bound $\Pr[A]$, which we do as follows
\begin{align}
\Pr[A]
& = 1 - \Pr[\overline{A}] \notag \\
& = 1 - \left( \Pr \left[ X_1 \le C \sigma \sqrt{\ln d} \right] \right)^d  \notag \\
& = 1 - \left(1 - \Pr\left[ X_1 > C \sigma \sqrt{\ln d} \right] \right)^d \notag \\
& \ge 1 - \exp\left(-d \cdot \Pr\left[X_1 \ge C \sigma \sqrt{\ln d} \right]\right) \notag \\
& \ge 1 - \exp\left(-d \exp\left(-\frac{C^2 \ln d}{2}\right) \frac{1}{\sqrt{2\pi}C \sqrt{\ln d}+2} \right) \notag \\
& = 1 - \exp\left(-\frac{d^{1-\frac{C^2}{2}}}{C \sqrt{2\pi \ln d}+2}\right) \label{equation:maximum-of-gaussians-lower-bound-4} \; .
\end{align}
where in the first inequality we used the elementary inequality $1 - x \le \exp(-x)$ valid for all $x \in \R$.

Since $C = \sqrt{2 - \frac{2 \ln \ln d}{\ln d}}$ we have $d^{1-\frac{C^2}{2}} = \ln d$. Substituting this into \eqref{equation:maximum-of-gaussians-lower-bound-4}, we get
\begin{equation}
\label{equation-maximum-of-gaussians-lower-bound-5}
\Pr[A] \ge 1 - \exp\left(-\frac{\ln d}{C \sqrt{2\pi \ln d}+2}\right) = 1 - \exp\left(-\frac{\sqrt{\ln d}}{C \sqrt{2\pi}+2}\right) \; .
\end{equation}
The function $C(d)$ is decreasing on the interval $[1,e^e]$, increasing on $[e^e, \infty)$, and $\lim_{d \to \infty} C(d) = \sqrt{2}$. From these properties
we can deduce that $C(d) \le \max\{C(2), \sqrt{2}\} \le 1.75$ for any $d \in [2,\infty)$. Therefore, $C\sqrt{2 \pi} + 2 \le 6.35$ and hence
\begin{equation}
\label{equation:maximum-of-gaussians-lower-bound-6}
\Pr[A] \ge 1 - \exp\left(-\frac{\sqrt{\ln d}}{6.35}\right) \; .
\end{equation}
Inequalities \eqref{equation:maximum-of-gaussians-lower-bound-3} and \eqref{equation:maximum-of-gaussians-lower-bound-6} together imply bound \eqref{equation:maximum-of-gaussians-lower-bound-1}.
Bound \eqref{equation:maximum-of-gaussians-lower-bound-2} is obtained from \eqref{equation:maximum-of-gaussians-lower-bound-1} by noticing that
\begin{align*}
& \sigma \left(1 - \exp\left(-\frac{\sqrt{\ln d}}{6.35}\right)\right) \left(\sqrt{2 \ln d - 2 \ln \ln d} +\sqrt\frac{2}{\pi}\right) -\sqrt{\frac{2}{\pi}} \sigma \\
& = \sigma \left(1 - \exp\left(-\frac{\sqrt{\ln d}}{6.35}\right)\right) \sqrt{2 \ln d - 2 \ln \ln d} - \exp\left(-\frac{\sqrt{\ln d}}{6.35}\right) \sqrt{\frac{2}{\pi}} \sigma \\
& \ge 0.1227 \cdot \sigma \sqrt{2 \ln d - 2 \ln \ln d} - 0.7 \sigma \\
& = 0.1227 \cdot \sigma \sqrt{\ln d} \cdot C(d) - 0.7 \sigma
\end{align*}
where we used that $\exp\left(-\frac{\sqrt{\ln d}}{6.35}\right) \le 0.8773$ for any $d \ge 2$.
Since $C(d)$ has minimum at $d = e^e$, it follows that $C(d) \ge C(e^e) = \sqrt{2 - \frac{2}{e}} \ge 1.1243$ for any $d \ge 2$.
\end{proof}

\section{Maximum of Random Walks}
\label{section:maximum-of-random-walks}

The general strategy for proving a lower bound on $\Exp\left[\max_{1 \le i \le d} Z^{(n)}_i \right]$
 is the same as in the previous section. The main task it to lower bound the
tail $\Pr[Z^{(n)} \ge x]$ of a symmetric random walk $Z^{(n)}$ of length
$n$.  Note that
$$
B_n = \frac{Z^{(n)} + n}{2}
$$
is a Binomial random variable $B(n,\frac{1}{2})$. We follow the same approach used in~\cite{nOrabona13}.
First we lower bound the tail $\Pr[B_n \ge k]$ with~\citet[Theorem 2]{McKay1989}.
\begin{lemma}[Bound on Binomial Tail]
\label{lemma:mckay}
Let $n,k$ be integers satisfying $n \ge 1$ and $\frac{n}{2} \le k \le n$. Define $x = \frac{2k - n}{\sqrt{n}}$. Then, $B_n \sim B(n,\frac{1}{2})$ satisfies
$$
\Pr \left[ B_n \ge  k \right] \ge \sqrt{n} \binom{n-1}{k-1} 2^{-n}
 \; \frac{1 - \Phi(x)}{\phi(x)} \; .
$$
\end{lemma}
We lower bound the binomial coefficient
$\binom{n-1}{k-1}$ using Stirling's approximation of the factorial. The lower
bound on the binomial coefficient will be expressed in terms of
Kullback-Leibler divergence between two Bernoulli distributions,
$\text{Bernoulli}(p)$ and $\text{Bernoulli}(q)$. Abusing notation somewhat,
we write the divergence as
$$
D(p\|q) = p \ln \left( \frac{p}{q} \right) + (1-p) \ln \left( \frac{1-p}{1-q} \right) \; .
$$
The result is the following lower bound on the tail of Binomial.
\begin{theorem}[Bound on Binomial Tail]
\label{theorem:binomial}
Let $n,k$ be integers satisfying $n \ge 1$ and $\frac{n}{2} \le k \le n$. Define $x = \frac{2k - n}{\sqrt{n}}$. Then, $B_n \sim B(n,\frac{1}{2})$ satisfies
$$
\Pr \left[ B_n \ge  k \right] \ge \frac{\exp\left(-n D \left(\frac{k}{n} \middle\| \frac{1}{2} \right)\right)}{\exp\left(\frac{1}{6}\right) \sqrt{2 \pi}} \; \frac{1 - \Phi(x)}{\phi(x)} \; .
$$
\end{theorem}
\begin{proof}
Lemma~\ref{lemma:mckay} implies that
$$
\Pr \left[ B_n \ge k \right] \ge \sqrt{n} \binom{n-1}{k-1} 2^{-n} \frac{1 - \Phi(x)}{\phi(x)}
$$
Since $k \ge 1$, we can write the binomial coefficient as
$$
\binom{n-1}{k-1} = \frac{k}{n} \binom{n}{k}
$$
We bound the binomial coefficient $\binom{n}{k}$ by using Stirling's formula for the factorial.
We use explicit upper and lower bounds due to~\cite{Robbins-1955} valid for any $n\ge 1$,
$$
\sqrt{2 \pi n} \left( \frac{n}{e} \right)^n < n! < \exp\left(\frac{1}{12}\right) \sqrt{2 \pi n} \left( \frac{n}{e} \right)^n \; .
$$
Using the Stirling's approximation, for any $1\le k \le n-1$,
\begin{align*}
\binom{n}{k}
& = \frac{n!}{k! (n-k)!} \\
& > \frac{\sqrt{2\pi n} \ n^n e^{-n}}{\sqrt{2\pi k} \ k^k e^{-k} e^{1/12} \cdot \sqrt{2\pi (n-k)} \ (n-k)^{n-k} e^{-(n-k)} e^{1/12}} \\
& = \frac{1}{\exp\left(\frac{1}{6}\right) \sqrt{2 \pi}} \left(\frac{n}{n-k}\right)^{n-k} \left(\frac{n}{k}\right)^k \sqrt{\frac{n}{k(n-k)}} \\
& = \frac{1}{\exp\left(\frac{1}{6}\right) \sqrt{2 \pi}} \ 2^n \exp\left(-n \cdot D\left(\frac{k}{n} \middle\| \frac{1}{2}\right)\right) \sqrt{\frac{n}{k(n-k)}}
\end{align*}
where in the equality we used the definition of $D(p\|q)$. Combining all the inequalities, gives
\begin{align*}
\Pr \left[ B_n \ge 2k - n \right]
& \ge \sqrt{n} \frac{k}{n} \frac{1}{\exp\left(\frac{1}{6}\right) \sqrt{2 \pi}} \ 2^n \exp\left(-n \cdot D\left(\frac{k}{n} \middle\| \frac{1}{2}\right)\right) \sqrt{\frac{n}{k(n-k)}} 2^{-n} \; \frac{1 - \Phi(x)}{\phi(x)} \\
& = \frac{1}{\exp\left(\frac{1}{6}\right) \sqrt{2\pi}} \exp\left(-n \cdot D\left(\frac{k}{n} \middle\| \frac{1}{2}\right)\right) \frac{1 - \Phi(x)}{\phi(x)} \sqrt{\frac{k}{n-k}} \\
& \geq \frac{1}{\exp\left(\frac{1}{6}\right) \sqrt{2\pi}} \exp\left(-n \cdot D\left(\frac{k}{n} \middle\| \frac{1}{2}\right)\right) \frac{1 - \Phi(x)}{\phi(x)}
\end{align*}
for $\frac{n}{2} \le k \le n - 1$. For $k=n$, we verify the statement of the theorem by direct substitution. The left hand side is $\Pr[B^{(n)} \ge n] = 2^{-n}$.
Since $e^{-nD(1\|\frac{1}{2})} = 2^{-n}$ and $x=\sqrt{n} \ge 1$, it's easy to see that the right hand side is smaller than $2^{-n}$.
\end{proof}

For $k=n/2 + xn$, the divergence $D \left( \frac{k}{n} \middle\| \frac{1}{2} \right) = D \left(\frac{1}{2} + x \| \frac{1}{2} \right)$ can be approximated by $2x^2$.
We define the function $\psi:[-\frac{1}{2},\frac{1}{2}] \to \R$ as 
$$
\psi(x) = \frac{D \left(\frac{1}{2}+x \middle\| \frac{1}{2} \right)}{2 x^2} \; .
$$
It is the ratio of the divergence and the approximation. The function $\psi(x)$ satisfies the following properties:
\begin{itemize}
\item $\psi(x) = \psi(-x)$
\item $\psi(x)$ is decreasing on $[-\frac{1}{2}, 0]$ and increasing on $[0, \frac{1}{2}]$
\item minimum value is $\psi(0) = 1$
\item maximum value is $\psi(\frac{1}{2}) = \psi(-\frac{1}{2}) = 2 \ln(2) \le 1.3863$
\end{itemize}
Using the definition of $\psi(x)$ and Theorem~\ref{theorem:binomial}, we have the following Corollary.
\begin{corollary}
\label{corollary:binomial}
Let $n \ge 1$ be a positive integer and let $t \in [1, \frac{n}{2} + 1]$ be a real number. Then $B_n \sim B(n, \frac{1}{2})$ satisfies
$$
\Pr \left[ B_n \ge \frac{1}{2} n + t - 1 \right] \ge \exp\left(-\frac{1}{6}\right) \exp\left(- 2 \psi\left(\frac{t}{n}\right) \frac{t^2}{n} \right) \frac{1}{\sqrt{2\pi} \frac{2 t}{\sqrt{n}} + 2} \; . 
$$
\end{corollary}
\begin{proof}
By Theorem~\ref{theorem:binomial} and Lemma~\ref{lemma:boyd}, we have
\begin{align*}
\Pr \left[ Z \ge  \frac{1}{2} n + t - 1 \right]
& = \Pr \left[ Z \ge \left\lceil \frac{1}{2} n + t - 1 \right\rceil \right] \\
& \ge \frac{\exp\left(-n D \left(\frac{\lceil \frac{1}{2}n + t - 1 \rceil}{n} \middle\| \frac{1}{2} \right)\right)}{\exp\left(\frac{1}{6}\right) \sqrt{2 \pi}} \cdot \frac{\pi}{\pi \frac{2\lceil \frac{1}{2}n + t - 1 \rceil - n}{\sqrt{n}} + \sqrt{2 \pi}} \\
& \ge \frac{\exp\left(-n D \left(\frac{\frac{1}{2}n + t}{n} \middle\| \frac{1}{2} \right)\right)}{\exp\left(\frac{1}{6}\right) \sqrt{2 \pi}} \cdot \frac{\pi}{\pi \frac{2 (\frac{1}{2}n + t) - n}{\sqrt{n}} + \sqrt{2 \pi}} \\
& = \frac{\exp\left(-n D \left(\frac{1}{2} + \frac{t}{n} \middle\| \frac{1}{2} \right)\right)}{\exp\left(\frac{1}{6}\right) \sqrt{2 \pi}} \cdot \frac{\pi}{\pi \frac{2t}{\sqrt{n}} + \sqrt{2 \pi}} \\
& = \exp\left(-\frac{1}{6}\right) \exp\left(- 2 \psi\left(\frac{t}{n}\right) \frac{t^2}{n} \right) \frac{1}{\sqrt{2\pi} \frac{2 t}{\sqrt{n}} + 2} \; . \qedhere
\end{align*}
\end{proof}

\begin{theorem}[Lower Bound on Maximum of Independent Symmetric Random Walks]
\label{theorem:maximum-of-random-walks}
Let $Z^{(n)}_1, Z^{(n)}_2, \dots, Z^{(n)}_d$ be $d$ independent symmetric random walks of length $n$. If $2 \le d \le \exp(\frac{n}{3})$ and $n \ge 7$,
\begin{align*}
\Exp \left[ \max_{1 \le i \le d} Z^{(n)}_i \right]
& \ge \frac{1 - \exp\left(-\frac{\sqrt{\ln d}}{3.1 \sqrt{2\pi}}\right)}{\sqrt{\psi\left(\frac{1.6\sqrt{\ln d}}{2 \sqrt{n}}\right)}} \sqrt{n} \left(\sqrt{2 \ln d - 2 \ln \ln d}-1\right) - \sqrt{n} \\
& \ge 0.09 \sqrt{n \ln d} - 2 \sqrt{n} \; .
\end{align*}
\end{theorem}
\begin{proof}
Define the event $A$ equal to the case that at least one of the $Z^{(n)}_i$ is greater or equal to $C \sqrt{n \ln d}-2$ where
$C = C(d,n) = \frac{1}{\sqrt{\psi\left(\frac{1.6\sqrt{\ln d}}{2 \sqrt{n}}\right)}}\sqrt{2-\frac{2 \ln \ln d}{\ln d}}$.

We upper and lower bound $C(d,n)$. Denote by $f(d)=\sqrt{2-\frac{2 \ln \ln d}{\ln d}}$
and notice that $C(d,n) = \frac{1}{\sqrt{\psi\left( \frac{1.6\sqrt{\ln d}}{2 \sqrt{n}}\right)}} f(d)$.  It
suffices to bound $f(d)$ and $\psi( \frac{1.6\sqrt{\ln d}}{2 \sqrt{n}})$.
We already know that $1 \le \psi(x) \le 2 \ln(2)$ for all $x \in [-\frac{1}{2}, \frac{1}{2}]$
and $\frac{1.6\sqrt{\ln d}}{2\sqrt{n}} \in [0,\frac{1}{2}]$ for $d \leq \exp(n/3)$.
The function $f(d)$ is decreasing on $(1,e^e]$,
increasing on $[e^e, \infty)$, and $\lim_{d \to \infty} f(d) = \sqrt{2}$. It has
unique minimum at $e^e$. Therefore, $f(d) \ge f(e^e) = \sqrt{2 - \frac{2}{e}} \ge 1.12$
for all $d \in (1,\infty)$.  Similarly, from unimodality of $f(d)$ we have that
$f(d) \le \max\{\sqrt{2}, f(2)\} = f(2) \le 1.6$ for all $d \in [2, \infty)$.
From this we can conclude that if $n \ge \ln d > 0$,
\begin{equation}
\label{equation:bound-on-constant}
0.95 \le \frac{f(e^e)}{\sqrt{2 \ln 2}} \le C(d,n) \le f(2) \le 1.6 \; .
\end{equation}
If $n \ge 7$ and $2 \le d \le \exp(n/3)$ this implies that
\begin{equation}
\label{equation:conditions}
1 < \frac{C \sqrt{n \ln d}}{2} < \frac{n}{2} + 1 \; .
\end{equation}

Recalling the definition of event $A$, we have
\begin{align*}
\Exp \left[ \max_{1 \le i \le d} Z^{(n)}_i \right]
& = \Exp \left[ \max_{1 \le i \le d} Z^{(n)}_i ~ \middle|~ A \right] \cdot \Pr[A] + \Exp \left[ \max_{1 \le i \le d} Z^{(n)}_i ~\middle|~ \overline{A} \right] \cdot \Pr \left[ \overline{A} \right] \\
& \ge \Exp \left[ \max_{1 \le i \le d} Z^{(n)}_i ~\middle|~ A \right] \cdot \Pr[A] + \Exp \left[ \max_{1 \le i \le d} Z^{(n)}_i ~\middle|~ \overline{A} \right] \cdot \Pr \left[ \overline{A} \right]\\
& \ge \Exp \left[ \max_{1 \le i \le d} Z^{(n)}_i ~\middle|~ A \right] \cdot \Pr[A] + \Exp\left[ Z^{(n)}_1 ~\middle|~ \overline{A} \right] \cdot \Pr \left[ \overline{A} \right] \\
& = \Exp \left[ \max_{1 \le i \le d} Z^{(n)}_i ~\middle|~ A \right] \cdot \Pr[A] + \Exp \left[ Z^{(n)}_1 ~\middle|~ Z^{(n)}_1 < C \sqrt{n \ln d} - 2 \right] \cdot \Pr \left[ \overline{A} \right]\\
& \ge \Exp \left[ \max_{1 \le i \le d} Z^{(n)}_i ~\middle|~ A \right] \cdot \Pr[A] + \Exp \left[ Z^{(n)}_1 ~\middle|~ Z^{(n)}_1 \le 0 \right] \cdot \Pr \left[ \overline{A} \right] & \text{(by \eqref{equation:conditions})} \\
& \ge (C \sqrt{n \ln d} - 2) \Pr[A] + \Exp \left[ Z^{(n)}_1 ~\middle|~ Z^{(n)}_1 \le 0 \right] (1 - \Pr[A]) \; .
\end{align*}

We lower bound $\Exp \left[ Z^{(n)}_1 ~\middle|~ Z^{(n)}_1 \le 0 \right]$. Using the fact that distribution of $Z^{(n)}_1$ is symmetric and has zero mean,
\begin{align*}
\Exp \left[ Z^{(n)}_1 ~\middle|~ Z^{(n)}_1 \le 0 \right]
& = \sum_{k=-n}^0 k \cdot \Pr[Z^{(n)}_1 = k ~|~ Z^{(n)}_1 \le 0] \\
& = \frac{1}{\Pr[Z^{(n)}_1 \le 0]} \sum_{k=-n}^0 k \cdot \Pr[Z^{(n)}_1 = k] \\
& \ge 2 \sum_{k=-n}^0 k \cdot \Pr[Z^{(n)}_1 = k] & \text{(by symmetry of $Z^{(n)}_1$)} \\
& = - \sum_{k=-n}^n |k| \cdot \Pr[Z^{(n)}_1 = k] & \text{(again, by symmetry of $Z^{(n)}_1$)} \\
& = - \Exp[|Z^{(n)}_1|] \\
& = - \Exp \left[ \sqrt{ \left( Z^{(n)}_1 \right)^2} \right] \\
& \ge - \sqrt{\Exp \left[ \left( Z^{(n)}_1 \right)^2 \right]} & \text{(by concavity of $\sqrt{\cdot}$)} \\
& = - \sqrt{\Var \left( Z^{(n)}_1 \right)} \\
& = -\sqrt{n}.
\end{align*}

Now let us focus on $\Pr[A]$. Note that $B_n = \frac{Z_1 + n}{2}$ is a binomial random variable with distribution $B(n,\frac{1}{2})$.
Similar to the proof of Theorem~\ref{theorem:maximum-of-gaussians}, we can lower bound $\Pr[A]$ as
\begin{align*}
\Pr[A]
& = 1 - \Pr \left[ \overline{A} \right] \\
& = 1 - \left(\Pr \left[ Z^{(n)}_1 < C \sqrt{n \ln d} - 2 \right] \right)^d \\
& = 1 - \left(\Pr \left[ B_n < \frac{n}{2} + \frac{C \sqrt{n \ln d}}{2} - 1 \right] \right)^d \\
& = 1 - \left(1 - \Pr\left[B_n \ge \frac{C \sqrt{n \ln d}}{2} +\frac{n}{2} - 1 \right] \right)^d \\
& \ge 1 - \exp\left(-d \cdot \Pr \left[ B_n \ge \frac{C \sqrt{n \ln d}}{2} +\frac{n}{2} - 1 \right] \right) & \text{(since $1 - x \le e^x$)} \\
& \ge 1 - \exp\left(-\frac{\exp\left(-\frac{1}{6}\right) d^{1-\frac{C^2}{2} \psi\left(\frac{C \sqrt{\ln d}}{2 \sqrt{n}}\right)}}{C \sqrt{2\pi} \sqrt{\ln d}+2}\right) & \text{(by Corollary \ref{corollary:binomial} and \eqref{equation:conditions})} \\
& \ge 1 - \exp\left(-\frac{\exp\left(-\frac{1}{6}\right) d^{1-\frac{C^2}{2} \psi\left(\frac{1.6 \sqrt{\ln d}}{2 \sqrt{n}}\right)}}{1.6 \sqrt{2\pi} \sqrt{\ln d}+2}\right) & \text{(by \eqref{equation:bound-on-constant}).}
\end{align*}
We now use the fact that $C=\frac{1}{\sqrt{\psi\left(\frac{1.6 \sqrt{\ln d}}{2 \sqrt{n}}\right)}}\sqrt{2- \frac{2\ln \ln d}{\ln d}}$ implies that $d^{1-\frac{C^2}{2} \psi\left(\frac{1.6 \sqrt{\ln d}}{2 \sqrt{n}}\right)}=\ln d$. Hence, we obtain
\begin{align*}
\Pr[A]
& \ge 1 - \exp\left(-\frac{\exp\left(-\frac{1}{6}\right) d^{1-\frac{C^2}{2} \psi\left(\frac{1.6 \sqrt{\ln d}}{2 \sqrt{n}}\right)}}{1.6 \sqrt{2\pi} \sqrt{\ln d}+2}\right) \\
& = 1 - \exp\left(-\frac{\exp\left(-\frac{1}{6}\right) \ln d}{1.6 \sqrt{2\pi} \sqrt{\ln d}+2}\right) \\
& \ge 1 - \exp\left(-\frac{\exp\left(-\frac{1}{6}\right) \sqrt{\ln d}}{2.6 \sqrt{2\pi}}\right)\\
& \ge 1 - \exp\left(-\frac{\sqrt{\ln d}}{3.1 \sqrt{2\pi}}\right)
\end{align*}
where in the last equality we used the fact that $\sqrt{2\pi} \sqrt{\ln d} > 2$ for $d\ge 2$. Putting all together, we have the stated bound.
\end{proof}

\section{Learning with Expert Advice}
\label{section:experts}

Learning with Expert Advice is an online problem where in each round $t$ an
algorithm chooses (possibly randomly) an action $I_t \in \{1,2,\dots,d\}$ and
then it receives losses of the actions $\ell_{t,1}, \ell_{t,2}, \dots,
\ell_{t,d} \in [0,1]$. This repeats for $n$ rounds.  The goal of the algorithm
is to have a small cumulative loss $\sum_{t=1}^n \ell_{t,I_t}$ of
actions it has chosen. The difference between the algorithm's loss and the loss of best fixed
action in hind-sight is called \emph{regret}.  Formally,
$$
\Regret^{(d)}(n) = \sum_{t=1}^n \ell_{t,I_t} - \min_{1 \le i \le d} \sum_{t=1}^n \ell_{t,i} \; .
$$
There are algorithms that given the number of rounds $n$ as an input achieve regret no more than $\sqrt{\frac{n}{2} \ln d}$
for any sequence of losses.

\begin{theorem}
Let $n \ge 7$ and $2 \le d \le \exp(\frac{n}{3})$. For any algorithm for learning with expert advice there exists a sequence
of losses $\ell_{t,i} \in \{0,1\}$, $1 \le i \le d$, $1 \le t \le n$, such that
$$
\Regret^{(d)}(n) \ge \frac{1 - \exp\left(-\frac{\sqrt{\ln d}}{3.1 \sqrt{2\pi}}\right)}{\sqrt{\psi\left(\frac{1.6 \sqrt{\ln d}}{2 \sqrt{n}}\right)}} \frac{\sqrt{n}}{2} \left(\sqrt{2 \ln d - 2 \ln \ln d}-1\right) - \frac{\sqrt{n}}{2} \; .
$$
\end{theorem}

\begin{proof}
Proceeding as in the proof of Theorem~3.7 in~\citep{Cesa-BianchiL06} we only need to show that
$$
\Regret^{(d)}{(n)} \ge \frac{1}{2} \Exp \left[ \max_{1 \le i \le d} Z^{(n)}_i \right]
$$
where $Z^{(n)}_1, Z^{(n)}_2, \dots, Z^{(n)}_d$ are independent symmetric random walks of length $n$. The theorem follows from Theorem~\ref{theorem:maximum-of-random-walks}.
\end{proof}

The theorem proves a non-asymptotic lower bounds, while at the same time
recovering the optimal constant of the asymptotic one in
\citet{Cesa-BianchiL06}.

\bibliographystyle{plainnat}
\bibliography{biblio}

\appendix

\section{Upper Bounds}
\label{section:upper-bounds}

We say that a random variable $X$ is \emph{$\sigma^2$-sub-Gaussian} (for some $\sigma \ge 0$) if
\begin{equation}
\label{equation:sigma-sub-gaussian}
\Exp \left[ e^{sX} \right] \le \exp\left( \frac{\sigma^2 s^2}{2} \right) \qquad \text{for all $s \in \R$} \; .
\end{equation}
It is straightforward to verify that $X \sim N(0,\sigma^2)$ is $\sigma^2$-sub-Gaussian. Indeed, for any $s \in \R$,
\begin{align*}
\Exp \left[ e^{sX} \right]
& = \int_{-\infty}^\infty \frac{1}{\sigma \sqrt{2\pi}} \exp\left( - \frac{x^2}{2\sigma^2} \right) e^{sx} dx \\
& = \exp\left( \frac{s^2\sigma^2}{2} \right) \int_{-\infty}^\infty \frac{1}{\sigma \sqrt{2\pi}} \exp\left( - \frac{(x - s\sigma^2)^2}{2\sigma^2} \right) dx \\
& = \exp\left( \frac{s^2\sigma^2}{2} \right) \; .
\end{align*}
We now show that a Rademacher random variable $Y$ (with distribution $\Pr[Y = +1] = \Pr[Y=-1] = \frac{1}{2}$)
is $1$-sub-Gaussian. Indeed, for any $s \in \R$,
\begin{align*}
\Exp \left[ e^{sY} \right]
= \frac{e^{s} + e^{-s}}{2}
= \frac{1}{2}\sum_{k=0}^\infty \frac{s^k}{k!}
+ \frac{1}{2}\sum_{k=0}^\infty (-1)^k \frac{s^k}{k!}
= \sum_{k=0}^\infty \frac{s^{2k}}{(2k)!}
\le \sum_{k=0}^\infty \frac{s^{2k}}{k! 2^k}
= \exp\left( \frac{s^2}{2} \right) \; .
\end{align*}
If $Y_1, Y_2, \dots, Y_n$ are independent $\sigma$-sub-Gaussian random variables, then $\sum_{i=1}^n Y_i$ is $(n\sigma^2)$-sub-Gaussian.
This follows from
$$
\Exp \left[ e^{s \sum_{i=1} Y_i} \right] = \prod_{i=1}^n \Exp \left[ e^{sY_i} \right] \; .
$$
This property proves that the symmetric random walk $Z^{(n)}$ of length $n$ is $n$-sub-Gaussian.

The upper bounds \eqref{equation:upper-bound-on-maximum-of-gaussians} and
\eqref{equation:upper-bound-on-maximum-of-random-walks} follow directly from
sub-Gaussianity of the variables involved and the following lemma.

\begin{lemma}[Maximum of sub-Gaussian random variables]
Let $X_1, X_2, \dots, X_d$ be (possibly dependent) $\sigma^2$-sub-Gaussian condition random variables.
Then,
$$
\Exp\left[ \max_{1 \le i \le d} X_i \right] \le \sigma \sqrt{2 \ln d} \; .
$$
\end{lemma}

\begin{proof}
For any $s > 0$, we have
\begin{align*}
\Exp \left[ \max_{1 \le i \le d} X_i \right]
& = \frac{1}{s} \Exp \left[ \max_{1 \le i \le d} \ln e^{s X_i} \right] \\
& \le \frac{1}{s} \ln \Exp \left[ \max_{1 \le i \le d} e^{s X_i} \right] \\
& \le \frac{1}{s} \ln \Exp \left[ \sum_{i=1}^d e^{s X_i} \right] \\
& = \frac{1}{s} \ln \sum_{i=1}^d \Exp \left[ e^{s X_i} \right] \\
& \le \frac{1}{s} \ln \left( d  \exp\left( \frac{\sigma^2 s^2}{2} \right) \right) \\
& = \frac{\ln d}{s} +  \frac{\sigma^2 s}{2} \; .
\end{align*}
Substituting $s=\frac{\sqrt{2 \ln d}}{\sigma}$ finishes the proof.
\end{proof}

\end{document}